\documentclass[letterpaper,11pt]{article}

\usepackage[utf8]{inputenc}
\usepackage{amssymb}
\usepackage[margin=1in]{geometry}
\usepackage{thmtools,thm-restate}
\usepackage{amsmath}
\usepackage{comment}
\usepackage{amssymb}
\usepackage{amsthm}
\usepackage{mathtools}
\usepackage{bbm}
\usepackage{amsfonts}
\usepackage{dsfont}
\usepackage{graphicx}
\usepackage{multirow}
\usepackage{color}
\usepackage{xcolor}
\usepackage{natbib}
\usepackage{hyperref}
\usepackage{graphicx}
\usepackage{subfigure}
\hypersetup{%
   breaklinks,%
   colorlinks=true,%
   linkcolor=black,%
   urlcolor=black,%
   citecolor=[rgb]{0,0,0.45}
}

\theoremstyle{plain}
\newtheorem{theorem}{Theorem}[section]

\newtheorem{lemma}[theorem]{Lemma}

\theoremstyle{definition}
\newtheorem{example}[theorem]{Example}
\newtheorem{claim}[theorem]{Claim}
\newtheorem{definition}[theorem]{Definition}
\newtheorem{assumption}[theorem]{Assumption}
\theoremstyle{remark}

\newcommand{\manipconj}[1]{c^*_{conj} \left(#1\right)} 
\newcommand{\manipseq}[1]{c^*_{seq} \left(#1\right)} 
\newcommand{\cost}{c} 
\newcommand{\reals}{\mathbb{R}}
\newcommand{\X}{\mathcal{X}}

\title{Sequential Strategic Screening}
\author{Lee Cohen\thanks{Toyota Technological Institute at Chicago (TTIC). Email: \href{mailto:lee@ttic.edu}{lee@ttic.edu}} \and Saeed Sharifi-Malvajerdi\thanks{Toyota Technological Institute at Chicago (TTIC). Email: \href{mailto:vakilian@ttic.edu}{saeed@ttic.edu}} \and Kevin Stangl\thanks{Toyota Technological Institute at Chicago (TTIC). 
Supported in part by the National Science Foundation under grant CCF-2212968 and by the Simons Foundation under the Simons Collaboration on the Theory of Algorithmic Fairness. Email: \href{mailto:kevin@ttic.edu}{kevin@ttic.edu}} \and Ali Vakilian\thanks{Toyota Technological Institute at Chicago (TTIC). Email: \href{mailto:vakilian@ttic.edu}{vakilian@ttic.edu}} \and Juba Ziani \thanks{Georgia Institute of Technology. Email: \href{mailto:jziani3@gatech.edu}{jziani3@gatech.edu}}}
\date{}
\begin{document}
\maketitle

\begin{abstract}
We initiate the study of strategic behavior in screening processes with \emph{multiple} classifiers. 
We focus on two contrasting settings: a ``conjunctive'' setting in which an individual must satisfy all classifiers simultaneously, and a sequential setting in which an individual to succeed must satisfy classifiers one at a time. In other words, we introduce the combination of \emph{strategic classification} with screening processes.

We show that sequential screening pipelines exhibit new and 
surprising behavior where individuals can exploit the sequential ordering of the tests to ``zig-zag'' between classifiers without having to simultaneously satisfy all of them. We demonstrate an individual can obtain a positive outcome using a limited manipulation budget even when far from the intersection of the positive regions of every classifier. 
Finally, we consider 
a learner whose goal is to design a sequential screening process that is robust to such manipulations, and provide a construction for the learner that optimizes a natural objective.
\end{abstract}
\newpage
\section{Introduction}\label{sec:intro}
Screening processes~\citep{arunachaleswaran2022pipeline, blum2022multi, cohen2019efficient} involve evaluating and selecting individuals for a specific, pre-defined purpose, such as a job, educational program, or loan application. 
These screening processes are generally designed to identify which individuals are qualified for a position or opportunity, often using multiple sequential classifiers or tests. For example, many hiring processes involve multiple rounds of interviews; university admissions can involve a combination of standardized tests, essays, or interviews. They have substantial practical benefits, in that they can allow a complex decision to be broken into a sequence of smaller and cheaper steps; this allows, for example, to split a decision across multiple independent interviewers, or across smaller and easier-to-measure criteria and requirements.  

Many of the decisions made by such screening processes are high stakes. For example, university admissions can affect an individual's prospects for their entire life. Loan decisions can have a long-term (sometimes even inter-generational) effect on a family's wealth or socio-economic status. When these decisions are high stakes, i.e. when obtaining a positive outcome is valuable or potentially life-changing or obtaining a negative outcome can be harmful, individuals may want to manipulate their features to trick the classifier into assigning them a positive outcome. 

In machine learning, this idea is known as strategic classification, and was notably introduced and studied by~\cite{bruckner2011stackelberg,hardt2016strategic}.
The current work aims to incorporate strategic classification within screening processes, taking a departure from the classical point of view in the strategic classification literature that focuses on a single classifier (see related work section). 

The key novel idea of our model of {\em strategic screening processes (or pipelines)}, compared to the strategic classification literature, comes from the fact that i) an individual has to pass and manipulate her way through \emph{several} classifiers, and ii) that we consider \emph{sequential} screening pipelines.

In a sequential screening pipeline, once an individual (also called \emph{Agent}) has passed a test or stage of this pipeline, she can ``forget'' about the said stage; whether or not she passes the next stage depends \textit{only on her performance in that stage}. For example, a job candidate that has passed the initial human resources interview may not need to worry about convincing that interviewer, and can instead expand her effort solely into preparing for the first technical round of interviews.
Alternatively, imagine a student `cramming' for a sequence of final exams, where one has a finite capacity to study that is used up over a week of tests. One wants to achieve a minimum score on each test, with a minimum of effort, by studying in between each test.

Our goal in this work is to examine how considering a pipeline comprised of a sequence of classifiers affects and modifies the way a strategic agent manipulates her features to obtain a positive classification outcome, and how a learner (which we primarily call the \emph{Firm}) should take this strategic behavior into account to design screening pipelines that are robust to such manipulation. 

We make a distinction between the following two cases: 1) the firm deploys its classifiers sequentially which we refer to as a {\em sequential screening process}; 2) the firm deploys a single classifier whose positive classification region is the intersection of the positive regions of the classifiers that form the pipeline which we sometimes refer to as \emph{simultaneous (or conjunctive) testing}---this single classifier is basically the {\em conjunction} or intersection of classifiers from the pipeline. The former corresponds to a natural screening process that is often used in practice and for which we give our main results, while the latter is primarily considered as a benchmark for our results for the sequential case. 

\begin{figure}[t]
\label{fig:zigzagjumps}
\centering
\includegraphics{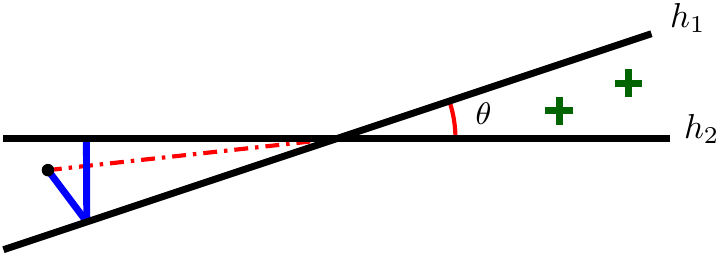}
\caption{Suppose the agent is the disqualified (i.e., placed in the negative region of the conjunctions of $h_1, h_2$) point.
A trivial manipulation strategy is to use the shortest {\em direct} path to the positive region, which is the dashed red path. However, the agent may also first manipulate slightly to pass $h_1$, then manipulate minimally again to pass $h_2$, as depicted with the blue solid path. This is what we call a zig-zag strategy.}
\centering
\end{figure}

\paragraph{Our Contributions.}
We show a perhaps surprising result: an agent can exploit the sequential nature of the screening process and move through the whole pipeline even when she started far from the intersection of the positive classification regions of all classifiers. 
In other words, the sequentiality of screening processes can \emph{improve} an agent's ability to manipulate her way through multiple classifiers compared to the simultaneous screening. We name the resulting set of strategies for such an agent in the sequential case \textit{``Zig-Zag" strategies}. In other words, whenever the agent does not manipulate straight to a point that is classified as positive by the conjunction of all classifiers, we call it a zig-zag strategy. An example of such a strategy that zig-zags between two classifiers is provided in Figure \ref{fig:zigzagjumps}.

In Figure \ref{fig:zigzagjumps}, since there is a small angle $\theta$ between the two tests, an agent at the bottom of the figure can zag right and then left as shown by the blue lines.  In this case, the agent is classified as positive in every single step, and by making $\theta$ arbitrarily small, will have arbitrarily lower total cost (e.g., the cumulative $\ell_2$ distance) compared to going directly to the intersection point of the classifiers. We provide concrete classifiers and an initial feature vector for such a case in Example \ref{exp:zigzag}.

In fact, in Section \ref{sec:example}  we show that for a given point, as $\theta$ goes to zero, the ratio between the total cost of the zig-zag strategy and the cost of 
going directly to the intersection can become arbitrarily large.
As we assume that conjunction of the classifiers captures the objective of the firm, using a pipeline can allow more disqualified people to get a positive outcome by manipulating their features. We show this in Figure~\ref{fig:2d-manipulation}: This figure shows the region of the agents space that can successfully manipulate to pass two linear tests in the two-dimensional setting, given a budget $\tau$ for manipulation. As shown by the figure, individuals in the green region of Figure~\ref{fig:2d-manipulation}.c can pass the tests in the sequential setting but would not be able to do so if they had to pass the tests simultaneously.

We further show how the optimal zig-zag strategy of an agent can be obtained computationally efficiently via a simple convex optimization framework in Section \ref{sec: algo} and provide a closed-form characterization of this strategy in the special case of $2$-dimensional features and a pipeline of exactly two classifiers in Section  \ref{subsec:closed-form}. 

In Section \ref{subsec-monotone} we consider a ``monotonicity" condition under which, agents prefer to use the simple strategy which passes all classifiers simultaneously in a single move and does not zig-zag between classifiers. 

 Finally, in Section~\ref{sec:conservative-defense}, we exhibit a defense strategy that maximizes true positives subject to not allowing any false positives. Interestingly, we show that under this strategy, deploying classifiers sequentially allows for a higher utility for the firm than using a conjunction of classifiers. 
 \begin{figure}[!t]
  \centering
  \subfigure($a$){\includegraphics[width=0.45\textwidth]{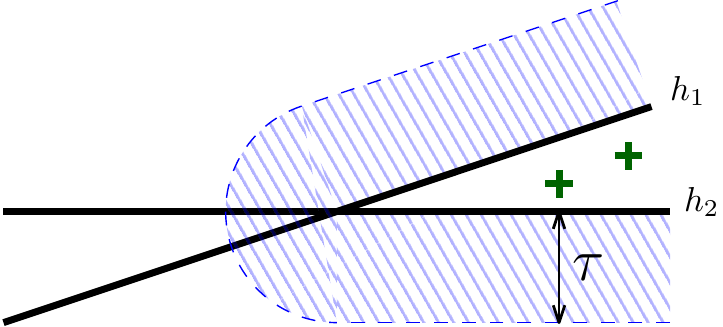}} 
  \subfigure($b$){\includegraphics[width=0.45\textwidth]{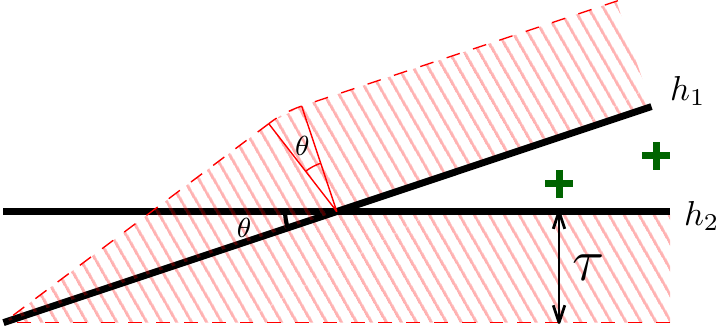}} 
  \subfigure($c$){\includegraphics[width=0.45\textwidth]{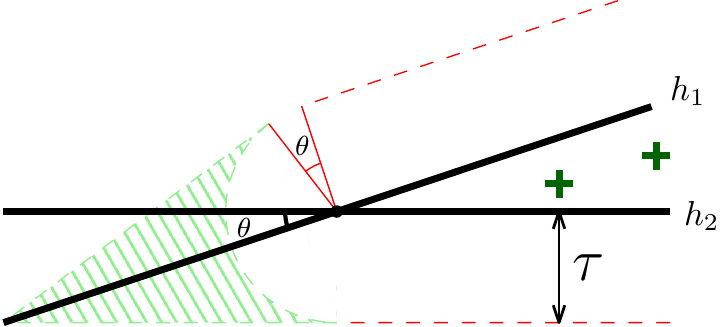}}
  \caption{Each agent has a manipulation budget of $\tau$ and the cost function is $\ell_2$ distance. Then,
  ($a$) shows the region of agents who afford to manipulate their feature vectors to pass both tests simultaneously, ($b$) shows the region of agents who afford to manipulate their feature vectors to pass the tests sequentially (i.e, first $h_1$, then $h_2$), and ($c$) shows the difference in these two regions.}
\label{fig:2d-manipulation}
\end{figure}

\paragraph{Related Work.} Our work inscribes itself at the intersection of two recent lines of work. The first one studies how strategic behavior affects decision-making algorithms (e.g. regression or classification algorithms), and how to design decision rules that take into account or dis-incentivize strategic behavior. This line of work is extensive and comprised of the works of~\citep{bruckner2011stackelberg, hardt2016strategic,kleinberg2020classifiers, randnoisestrat, miller2020strategic,strat1,strat2,strat3,strat4,strat5,strat6,strat8,strat9,strat10,strat11,strat12,strat13,strat14,strat15,strat16,stratperceptron,tang2021linear,hu2019disparate,socialcost18,performative2020,stratindark,randnoisestrat,gameimprove,bechavod2021gaming, bechavod2022information,shavit2020causal,dong2018strategic,chen2020learning,harris2021stateful}.

The second line of work is separate and aims to understand how decisions compose and affect each other in decision-making and screening pipelines~\citep{cohen2019efficient,BowerKNSVV17,blum2022multi,arunachaleswaran2022pipeline, dwork2020individual, faircomp}. These works studies settings in which \emph{multiple} decisions are made about an individual or an applicant.

However, and to the best of our knowledge, there is little work bringing these two fields together and studying strategic behavior in the context of decision \emph{pipelines} comprised of \emph{multiple} classifiers. This is where the contribution of the current work lies. 

\section{Our Model}
Formally, individuals (or agents) are represented by a set of features $x \in \mathcal{X}$, where $\mathcal{X} \subseteq \mathbb{R}^d$, for $d \ge 1$. 
The firm has a fixed sequence of binary tests or classifiers $h_1, h_2, \dots, h_k:\mathcal{X} \rightarrow \{ 0, 1 \}$ that are deployed to select qualified individuals while screening out unqualified individuals. Here, an outcome of $1$ (positive) corresponds to an acceptance, and an outcome of $0$ (negative) corresponds to a rejection. Once a person is rejected by a test they leave the pipeline. 

In the whole paper, we assume that the classifiers are linear and defined by half-spaces; i.e. $h_{i}(x) = 1 \iff w^{\top}_{i} x \ge b_i $ for some vector $w_i \in \reals^d$ and real threshold $b_i \in \reals$. Equivalently, we often write $h_i(x) = \mathbbm{1} \left[w^{\top}_{i} x \ge b_i\right]$.\footnote{While more general classes of classifiers could be considered, linear classifiers are a natural starting point to study strategic classification. This linearity assumption arises in previous work, e.g.~\citep{kleinberg2020classifiers,tang2021linear,gameimprove} to only name a few.} 

In this work we assume that the true qualifications of individuals are determined by the conjunction of the classifiers adopted by the firm in the pipeline, i.e. an agent $x$ is qualified if and only if $h_i (x) = 1$ for all $i$. In other words, the firm has designed a pipeline that makes no error in predicting individuals' qualifications \emph{absent strategic behavior}.

However, in the presence of strategic behavior, individuals try to manipulate their feature vectors to become positively classified by the classifiers simply because they receive a positive utility from  a positive outcome. Similar to prior works, throughout this work, we assume a ``white box" model meaning agents know the parameters for each classifier. More precisely, the firm commits to using a sequential screening process consisting of classifiers $h_1, h_2 \ldots h_k$, and each agent knows the parameters of each hypothesis, the order of the tests, her own feature value $x$, and the cost to manipulate to any other point in the input space.

An agent's cost function is modeled by a function $c: \X \times \X \to \mathbb{R}_{\ge 0}$ that takes two points $x, \hat{x}$ and outputs the cost of moving from $x$ to $\hat{x}$. One can think of $x$ as the initial feature vector of an agent and $\hat{x}$ as the manipulated features. In the sequential setting that we consider, we take the cost of manipulation to be the cumulative cost across every single manipulation. In particular, for a manipulation path $x^{(0)} \to x^{(1)} \to x^{(2)} \to \ldots \to x^{(k)}$ taken by an agent whose true feature values are $x^{(0)}$, the cost of manipulation is given by $\sum_{i=1}^k c(x^{(i-1)}, x^{(i)})$. We assume such manipulations do not change nor improve one's true qualifications\footnote{E.g., in a loan application, such manipulations could be opening a new credit card account: doing so may temporarily increase an agent's credit score, but does not change anything about an agent's intrinsic financial responsibility and ability to repay the loan.} and we discuss how the firm mitigates this effect of manipulation. 

In turn, the firm's goal is to have an accurate screening process whose predictions are as robust to and unaffected by such strategic: the firm modifies its classifiers $h_1, \cdots, h_k$ to $\Tilde{h}_1, \cdots, \Tilde{h}_k$ so that the output of $\Tilde{h}_1, \cdots, \Tilde{h}_k$ on manipulated agents' features can identify the qualified agents optimally with respect to a given ``accuracy measure"; we will consider two such measures in Section~\ref{sec:defense}.   

\subsection{Agent's Manipulation}
We proceed by formally defining the minimal cost of manipulation, which is the minimal cost an agent has to invest to pass all classifiers, and the best response of an agent for both sequential and simultaneous testing. 
\begin{definition}[Manipulation Cost: Sequential]
    Given a sequence of classifiers $h_1, \ldots, h_k$, a global cost function $c$, and an agent $x^{(0)} \in \mathcal{X}$, the manipulation cost of an agent in the sequential setting is defined as the minimum cost incurred by her to pass all the classifiers sequentially, i.e.,
    \begin{align*}
    \manipseq{x^{(0)},\{h_1,\ldots,h_k\}}
    = \min_{x^{(1)}, \ldots, x^{(k)} \in \mathcal{X}}~~~&\sum_{i=0}^{k-1} c(x^{(i)}, x^{(i+1)}) 
    \\\text{s.t.}~~~~~~~~~&h_i(x^{(i)}) = 1~~\forall i \in [k].
    \end{align*}
    The \textit{best response} of $x^{(0)}$ to the sequential testing $h_1,\ldots,h_k$ is the path $x^{(1)},\ldots,x^{(k)}$ that minimizes the objective.
\end{definition}

\begin{definition}[Manipulation Cost: Conjunction or Simultaneous]
    Given a set of classifiers $\{h_1, \ldots, h_k\}$, a global cost function $c$, and an agent $x$, the manipulation cost of an agent in the conjunction setting is defined as the minimum cost incurred by her to pass all the classifiers at the same time, i.e.,
    \begin{align*}
    \manipconj{x,\{h_1,\ldots,h_k\}} = \min_{z \in \mathcal{X}}&~~~c(x, z)
    \\\text{s.t.}&~~~h_i(z)=1 \ \forall i \in [k].
    \end{align*}
    
    The \textit{best response} of $x$ to the conjunction of $h_1\ldots,h_k$ is the $z$ that minimizes the objective.
\end{definition}

\section{Best Response of Agents in a Screening Process with Oblivious Defender}\label{sec:bestRes}
In this section, we study the manipulation strategy of an agent. In particular, we present algorithms to compute optimal manipulation strategies efficiently.
We make the following assumption on the cost function in most of the section, unless explicitly noted otherwise:

\begin{assumption}\label{asst: l2_cost}
The cost of moving from $x$ to $\hat{x}$ is given by $c(x,\hat{x}) = \Vert \hat{x} - x \Vert_2$, where $\Vert . \Vert_2$ denotes the standard Euclidean norm.
\end{assumption}
\subsection{Optimal Strategies in the Conjunction Case}
\label{sec: optstragcojn}

As a warm-up to our zig-zag strategy in Section \ref{sec: algo}, we first consider the optimal strategy for our benchmark, which is the case of the simultaneous conjunction of $k$ classifiers.
In the case where agents are supposed to pass a collection of linear classifiers simultaneously, the best response of an agent $x \in \mathbbm{R}^d$ is given by solving the following optimization problem
\begin{align}
\begin{split}\label{eq: convex_optim}
\min_{z}~& c(x,z)\\
\text{s.t.}~~~&w_i^\top z \geq b_i \ \forall i \in [k].
\end{split}
\end{align}
which is a convex program as long as $c$ is convex in $z$. 

In the special case in which $d = 2$ and $k = 2$, i.e. when feature vectors are two-dimensional and an agent must be positively classified by the conjunction of two linear classifiers $h_1 (x) = \mathbbm{1} (w_1^\top x \ge b_1 )$ and $h_2 (x) = \mathbbm{1} (w_2^\top x \ge b_2 )$, we provide a closed form characterization of an agent's strategy. 

We assume that the two classifiers are \emph{not} parallel to each other because if $w_2 = kw_1$ for some $k \in \mathbb{R}$, then one can show that either the acceptance regions of $h_1$ and $h_2$ do not overlap, or the optimal strategy of an agent is simply the orthogonal projection onto the intersection of the acceptance regions of $h_1$ and $h_2$.

We further assume, without loss of generality, that $b_1 = b_2 = 0$ because if either $b_1$ or $b_2$ is nonzero, one can use the change of variables $x' \triangleq x + s$ to write the classifiers as $h_1 (x') = \mathbbm{1} (w_1^\top x' \ge 0)$ and $h_2 (x') = \mathbbm{1} (w_2^\top x' \ge 0)$. Here $s$ is the solution to $\{w_1^\top s = -b_1, w_2^\top s = -b_2\}$.

For any $w \in \mathbb{R}^2$ with $\Vert w \Vert_2 = 1$, let $P_w (x)$ and $d_w (x)$ be the orthogonal projection of $x$ onto the region $\{y \in \mathbb{R}^2 : w^\top y \ge 0 \}$, and its orthogonal distance to the same region, respectively. We have
\begin{align*}
& P_w (x) \triangleq
\begin{cases}
    x & \text{if }w^\top x \ge 0 \\ x - (w^\top x) w & \text{if }w^\top x < 0
\end{cases}, \\
&d_w (x) \triangleq
\begin{cases}
    0 & \text{if }w^\top x \ge 0 \\ |w^\top x| & \text{if }w^\top x < 0
\end{cases}.
\end{align*}
Given this setup, the best response characterization of an agent $x$ can be given as follows. If $h_1 (x) = h_2 (x) = 1$ then $z=x$. Otherwise, the best response is either the orthogonal projection onto the acceptance region of $h_1$ or $h_2$, or moving directly to the intersection of the classifiers ($\Vec{0}$):
\begin{enumerate}
    \item If $h_1 (P_{w_2} (x)) = 1$, then $z = P_{w_2} (x)$ and the cost of manipulation is $\manipconj{ x^{(0)} ,\{h_1,h_2\}} = d_{w_2} (x)$.
    \item If $h_2 (P_{w_1} (x)) = 1$, then $z = P_{w_1} (x)$ and the cost of manipulation is $\manipconj{ x^{(0)} ,\{h_1,h_2\}} = d_{w_1} (x)$.
    \item if $h_1 (P_{w_2} (x)) = h_2 (P_{w_1} (x)) = 0$ then $z = \Vec{0}$ and the cost of manipulation is $\manipconj{ x^{(0)} ,\{h_1,h_2\}} = \Vert x\Vert_2$.
\end{enumerate}

Given a budget $\tau$, agents who can manipulate with a cost of at most $\tau$ to pass the two tests simultaneously, i.e. $\{x^{(0)}: \manipconj{ x^{(0)} ,\{h_1,h_2\}} \le \tau \} $ is highlighted in Figure~\ref{fig:2d-manipulation}.a.

\subsection{A Zig-Zag Manipulation on Sequential Classification Pipelines}\label{sec:example}
Here, we make the observation that the sequential nature of the problem can change how an agent will modify her features in order to pass a collection of classifiers, compared to the case when said classifiers are deployed simultaneously. We illustrate this potentially counter-intuitive observation via the following simple example: 
\begin{example}\label{exp:zigzag}
Consider a two-dimensional setting. Suppose an agent going up for classification has an initial feature vector $x = (0,0)$. Suppose the cost an agent faces to change her features from $x$ to a new vector $\hat{x}$ is given by $\Vert \hat{x} - x\Vert_2$. Further, imagine an agent must pass two classifiers: $h_1(x) = \mathbbm{1}\left\{4 x_2 - 3 x_1 \geq 1\right\}$, and $h_2(x) = \mathbbm{1} \left\{x_1 \geq 1\right\}$, where $x_i$ is the $i-$th component of $x$.

It is not hard to see, by triangle inequality, that if an agent is facing a conjunction of $h_1$ and $h_2$, an agent's cost is minimized when $\hat{x} = (1,1)$ (this is in fact the intersection of the decision boundaries of $h_1$ and $h_2$), in which case the cost incurred by an agent is $\sqrt{1 + 1} = \sqrt{2}$ (see the red manipulation in Figure~\ref{fig:example}). 

\begin{figure}
    \centering
    \includegraphics{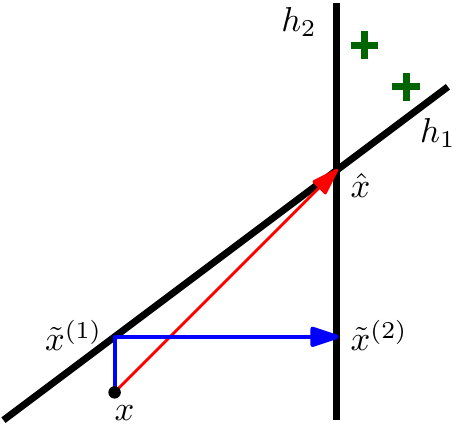}
    \caption{An example for a zig-zag strategy being better for an agent that starts at $x$ in the sequential case than moving in a single step. Here, an agent would prefer to first manipulate to $\tilde{x}^{(1)}$ then to $\tilde{x}^{(2)}$ (the blue arrows) instead of straightforwardly moving from $x$ to $\hat x$ as would be optimal in the conjunction case (the red arrow). 
    }
    \label{fig:example}
\end{figure}

However, if the classifiers are offered sequentially, i.e. $h_1$ then $h_2$, consider the following feature manipulation: first, the agent sets $\tilde{x}^{(1)} = (0,1/4)$, in which case she passes $h_1$ and incurs a cost of $1/4$. Then, the agent sets $\tilde{x}^{(2)} = (1,1/4)$; the cost to go from $\tilde{x}^{(1)}$ to $\tilde{x}^{(2)}$ is $\Vert_2 (1,1/4) - (0,1/4) \Vert = 1$ (see the blue manipulation in Figure~\ref{fig:example}). In turn, the total cost of this manipulation to pass (i.e., get a positive classification on) both classifiers is at most $1 + 1/4 = 5/4$, and is always better than the $\sqrt{2}$ cost for the conjunction of classifiers! 
\end{example}

Intuitively, here, the main idea is that in the ``conjunction of classifiers'' case, an agent must manipulate her features a single time in a way that satisfies all classifiers at once. However, when facing a sequence of classifiers $h_1, \ldots, h_k$, once an agent has passed classifier $h_{i-1}$ for any given $i$, it can ``forget'' classifier $h_{i-1}$ and manipulate its features to pass $h_i$ while \emph{not being required to pass $h_{i-1}$} anymore. In turn, the potential manipulations for an agent in the sequential case are less constrained than in the conjunction of classifiers case. This result is formalized below: 

\begin{restatable}{claim}{claimConjUpBoundsSeq}\label{claim:conjBoundSeq}
Let $h_1, \ldots, h_k$ be a sequence of $k$ linear classifiers. For any agent with initial feature vector $x \in \mathbb{R}^d$ ($d\ge 1$), 
\[
\manipconj{x,\{h_1,\ldots,h_k\}} \geq \manipseq{x,\{h_1,\ldots,h_k\}}.
\]
\end{restatable}
\begin{proof}
Let $\cost$ be the agent's cost function. Let $\hat{x}$ be a vector such that $h_i(\hat{x}) = 1$ for all $i \in [k]$, and such that $\cost(x,\hat{x}) \leq \tau$ where $\tau$ is the manipulation budget available to the agent. Since $\hat{x}$ satisfies $h_i(\hat{x}) = 1$ for all $i \in [k]$, the feature modification $x \to \hat{x}$  gives a positive classification outcome to the agent in the sequential case. Further, the cost of this manipulation is $\cost(x,\hat{x}) + 0 + \ldots + 0 = \cost(x,\hat{x})$. In turn, for any feasible one-shot manipulation that passes all classifiers in the conjunctive case, there exists a feasible sequential manipulation that passes all classifiers in the sequential case which could be of a lower cost; this concludes the proof. 
\end{proof}
Intuitively, the above claim follows from the observation that any best response solution to the conjunction case in particular still passes all classifiers and has the same cost in the sequential case.

However, there can be a significant gap between how much budget an agent needs to spend in the conjunctive versus in the sequential case to successfully pass all classifiers (for illustration, see Figure~\ref{fig:2d-manipulation}). In fact, we show below that the multiplicative gap between the conjunctive and sequential manipulation cost can be unbounded, even in the two-dimensional setting:
\begin{lemma}
\label{lem:costgap}
Consider $d = 2$. For any constant $M>0$, there exists two linear classifiers $h_1$ and $h_2$ and an initial feature vector $x^{(0)}$ such that 
\[
\frac{\manipconj{x^{(0)},\{h_1,h_2\}} }{\manipseq{x^{(0)},\{h_1,h_2\}}} \geq M.
\]
\end{lemma}

\begin{proof}
Pick $x^{(0)} = (0,0)$. Let $\gamma > 0$ be a real number. Consider $h_1(x) = \mathbbm{1} \left\{\frac{x_1}{\gamma} + x_2 \geq 1 \right\}$ and $h_2(x) = \mathbbm{1} \left\{\frac{x_1}{\gamma} - x_2 \geq 1 \right\}$. Let $\hat{x}$ be the agent's features after manipulation. To obtain a positive classification outcome, the agent requires both  $\hat{x}_1 \geq \gamma (1 - \hat{x}_2)$ and $\hat{x}_1 \geq \gamma (1 + \hat{x}_2)$. Since one of $1 - \hat{x}_2$ or $1 + \hat{x}_2$ has to be at least $1$, this implies $\hat{x}_1 \geq \gamma$. In turn, $c(x,\{h_1,h_2\}) = \Vert \hat{x}\Vert \geq \gamma$. 

However, in the sequential case, a manipulation that passes $h_1$ is to set $x^{(1)} = (0,1)$. Then a manipulation that passes $h_2$, starting from $x^{(1)}$, is to set $x^{(2)} = (0,-1)$. The total cost is $\Vert (0,1) - (0,0) \Vert + \Vert (0,-1) - (0, 1)\Vert = 1 + 2 =3$. In particular, 
\[
\frac{\manipconj{x,\{h_1,\ldots,h_k\}} }{\manipseq{x,\{h_1,\ldots,h_k\}}} \geq \gamma/3.
\]
The result is obtained by setting $\gamma = 3M$.
\end{proof}

\subsection{An Algorithmic Characterization of an agent's Optimal Strategy in the Sequential Case}\label{sec: algo}
In this section, we show that in the sequential setting, an agent can compute her optimal sequences of manipulations efficiently. Consider any initial feature vector $x^{(0)} \in \mathbb{R}^d$ for an agent. Further, suppose an agent must pass $k$ linear classifiers $h_1,\ldots,h_k$. For $i \in [k]$, we write once again $h_i (x) = \mathbbm{1}[w_i^\top x\geq b_i]$ the $i$-th classifier that an agent must get a positive classification on. Here and for this subsection only, we relax our assumption on the cost function to be more general, and not limited to $\ell_2$ costs:

\begin{assumption}\label{convex_cost}
The cost $c(x,\hat{x})$ of moving from feature vector $x$ to feature vector $\hat{x}$ is convex in $(x,\hat{x})$.
\end{assumption}

This is a relatively straightforward and mild assumption; absent convexity, computing the best feature modifications for even a single step can be a computationally intractable problem. The assumption covers but is not limited to a large class of cost functions of the form $c(x,\hat{x}) = \Vert \hat{x} - x \Vert$, for \emph{any} norm $\Vert . \Vert$. It can also encode cost functions where different features or directions have different costs of manipulation; an example is $c(x,\hat{x}) = \left(\hat{x} - x \right)^\top A \left(\hat{x} - x \right)$ where $A$ is a positive definite matrix, as used in~\citep{ shavit2020causal,bechavod2022information}.

In this case, an agent's goal, starting from her initial feature vector $x^{(0)}$, is to find a sequence of feature modifications $x^{(1)}$ to $x^{(k)}$ such that: 1) for all $i \in [k]$, $h_i(x^{(i)}) = 1$. I.e., $x^i$ passes the $i$-th classifier; and 2) the total cost $\sum_{i=1}^k c(x^{(i-1)},x^{(i)})$ of going from $x^{(0)} \to x^{(1)} \to x^{(2)} \to \ldots \to x^{(k)}$ is minimized. This can be written as the following optimization problem: 
\begin{align}
\begin{split}\label{eq: convex_optim2}
\min_{x^{(1)},\ldots,x^{(k)}}~&\sum_{i=1}^k c(x^{(i-1)},x^{(i)})\\
\text{s.t.}~~~&w_i^\top x^{(i)} \geq b_i \ \forall i \in [k].
\end{split}
\end{align}

\begin{claim}
Program~\eqref{eq: convex_optim2} is convex in $(x^{(1)},\ldots,x^{(k)})$.
\end{claim}

In turn, we can solve the problem faced by an agent's computationally efficiently, through standard convex optimization techniques. 

\subsection{A Closed-Form Characterization in the 2-Classifier, 2-Dimensional Case}
\label{subsec:closed-form}
We now provide closed-form characterization of an agent's best response in the sequential case, under the two-dimensional two-classifier ($d=k=2$) setting that we considered in Section~\ref{sec: optstragcojn}. Here, we take the cost function to be the standard Euclidean norm, i.e. $c(x,\hat{x}) = \Vert \hat{x} - x \Vert_2$, as per Assumption~\ref{asst: l2_cost}.

\begin{restatable}{theorem}{twoDcharacthm}\label{thm: 2d_charac}
Consider two linear classifiers $h_1 (x) = \mathbbm{1} (w_1^\top x \ge 0)$ and $h_2 (x) = \mathbbm{1} (w_2^\top x \ge 0)$ where $\Vert w_i \Vert_2 = 1$ for $i \in \{1,2\}$ and an agent $x^{(0)} \in \mathbb{R}^2$ such that $h_1 (x^{(0)}) = 0$ and $h_2 (P_{w_1} (x^{(0)})) = 0$. Let $0 < \theta < \pi$ be the angle between (the positive region of) the two linear classifiers; i.e. $\theta$ is the solution to $\cos \theta = - w_1^\top w_2 $. Then:
\begin{enumerate}
\item If $|\tan \theta| > \Vert P_{w_1} (x^{(0)}) \Vert_2 / d_{w_1} (x^{(0)})$, then the best response for an agent is to pick 
\[
x^{(2)} = x^{(1)} = \Vec{0}.
\]
In this case, the cost of manipulation is 
\[
\manipseq{ x^{(0)} ,\{h_1,h_2\}} = \Vert x^{(0)} \Vert_2. 
\]
\item If $|\tan \theta| \le \Vert P_{w_1} (x^{(0)}) \Vert_2 / d_{w_1} (x^{(0)})$, then the best response is given by
\[
x^{(1)}  = \left( 1 - \frac{d_{w_1} (x^{(0)})}{\Vert P_{w_1} (x^{(0)}) \Vert_2} |\tan \theta| \right) P_{w_1} (x^{(0)})
\] 
and $x^{(2)} = P_{w_2} (x^{(1)})$, and the cost of manipulation is given by 
\begin{align*}
&\manipseq{ x^{(0)} ,\{h_1,h_2\}} 
= d_{w_1} (x^{(0)}) | \cos \theta | + \Vert P_{w_1} (x^{(0)})\Vert_2  \sin \theta.
\end{align*}
\end{enumerate}
\end{restatable}
\begin{proof}
\begin{figure}
    \centering
    \includegraphics{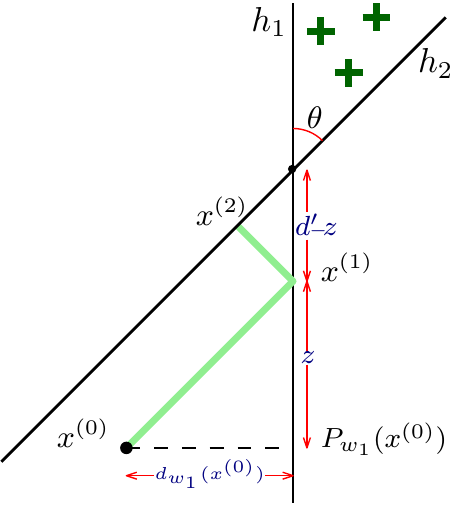}
    \caption{Illustration of the reduction from the optimization problem in Equation~\ref{eq:bestresponse} to the one in Equation~\ref{eq:reduction}.
    }
    \label{fig:zigzag-example}
\end{figure}
Given classifiers $h_1$ and $h_2$, the best response of an agent $x^{(0)}$ is a solution to the following optimization problem, as noted in Section~\ref{sec: algo}:
\begin{align*}
&\manipseq{ x^{(0)} ,\{h_1,h_2\}}
= \min_{x^{(1)}, x^{(2)}} \left\{ \Vert x^{(0)} - x^{(1)} \Vert_2 + \Vert x^{(1)} - x^{(2)} \Vert_2 : w_1^\top x^{(1)} \ge 0, w_2^\top x^{(2)} \ge 0 \right\}
\end{align*}
First, we remark that given any $x^{(1)}$, the optimal choice of $x^{(2)}$ is the orthogonal projection of $x^{(1)}$ on classifier $f_2$. Therefore, the best response can be written as:
\begin{equation}
\label{eq:bestresponse}
\manipseq{ x^{(0)} ,\{h_1,h_2\}} = \min_{x^{(1)} \in \mathbb{R}^2} \left\{ \Vert x^{(0)} - x^{(1)} \Vert_2 + d_{w_2} \left(x^{(1)}\right) : w_1^\top z \ge 0 \right\}
\end{equation}
To simplify notations, we will denote $x \triangleq x^{(0)}$. Under the assumptions of the theorem (more specifically, $h_1 (x) = 0$ and $h_2 (P_{w_1} (x)) = 0$), Equation~\eqref{eq:bestresponse} can be rewritten as an optimization over a one-dimensional variable:
\begin{equation}\label{eq:reduction}
\min_{0 \le z \le d'_{w_1} (x)} \left\{ g(z) \triangleq \sqrt{d_{w_1}^2 (x) + z^2} + (d'_{w_1} (x) - z) \sin \theta \right\}
\end{equation}
where $d'_{w_1} (x) \triangleq \Vert P_{w_1} (x)\Vert_2$ -- see Figure~\ref{fig:zigzag-example} for a graphical justification of this rewriting. 
Note that $g(z)$ achieves its minimum either at the boundaries or at the point where $g'(z) = 0$. Therefore, we have that the minimum is one of the following:
\begin{align*}
& z = 0 \Longrightarrow  g(z) = d_{w_1} (x) + d'_{w_1} (x) \sin \theta \\
&z = d'_{w_1} (x) \Longrightarrow  g(z) = \sqrt{d_{w_1}^2 (x) + d^{'2}_{w_1} (x)} = \Vert x\Vert_2 \\
& z = d_{w_1} (x) | \tan \theta| \Longrightarrow g(z) = d_{w_1} (x) | \cos \theta | + d'_{w_1} (x) \sin \theta \ (g'(z) = 0)
\end{align*}
If $d_{w_1} (x) | \tan \theta| \le d'_{w_1} (x)$, then an application of Cauchy-Schwarz inequality implies that $z=d_{w_1} (x) | \tan \theta|$ is the minimzer. Therefore, if $|\tan \theta| > d'_{w_1} (x) / d_{w_1} (x)$, the minimizer is $z^\star =  d'_{w_1} (x)$, meaning $x^{(2)} = x^{(1)} = \Vec{0}$, and that
\[
\manipseq{ x ,\{h_1,h_2\}} = \Vert x\Vert_2
\]
and if $|\tan \theta| \le d'_{w_1} (x) / d_{w_1} (x)$, the minimizer is $z^\star =  d_{w_1} (x) | \tan \theta|$ which implies
\[
x^{(1)}  = \left( 1 - \frac{d_{w_1} (x^{(0)})}{\Vert P_{w_1} (x^{(0)}) \Vert_2} |\tan \theta| \right) P_{w_1} (x^{(0)})
\] 
and $x^{(2)} = P_{w_2} (x^{(1)})$, and that
\[
\manipseq{ x ,\{h_1,h_2\}} = d_{w_1} (x) | \cos \theta | + d'_{w_1} (x) \sin \theta
\]
Therefore, putting the two cases together,
\begin{align*}
& \manipseq{ x ,\{h_1,h_2\}} = 
\begin{cases}
\Vert x\Vert_2 & \text{if } |\tan \theta| > d'_{w_1} (x) / d_{w_1} (x)\\
d_{w_1} (x) | \cos \theta | + d'_{w_1} (x) \sin \theta & \text{if } |\tan \theta| \le d'_{w_1} (x) / d_{w_1} (x)
\end{cases}
\end{align*}
\end{proof}
\begin{figure}[t]
    \centering
    \includegraphics{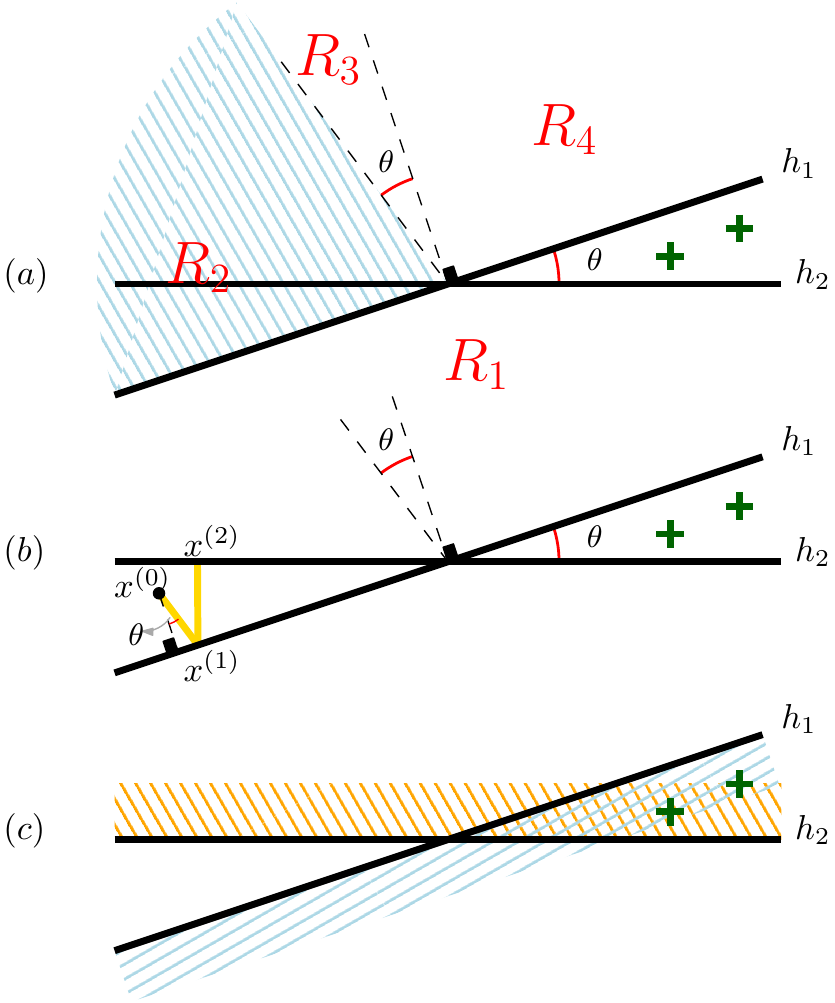}
    \caption{(a) Different cases for how agents best respond: agents in $R_1$ stay at their location to pass the first test and project onto $h_2$ to pass the second. Agents in $R_2$ deploy a zig-zag strategy. Agents in $R_3$ move to the intersection of $h_1$ and $h_2$. Agents in $R_4$ project onto $h_1$. (b) Geometric characterization of the zig-zag strategy: the line passing through $x^{(0)}$ and $x^{(1)}$ has angle $\theta$ with the line perpendicular to $h_1$. (c) This figure highlights the positive regions of $h_1$, $h_2$, and their intersection.}
    \label{fig:zigzag-region}
\end{figure}

First, note that once the first feature modification has happened and an agent has passed classifier $h_1$ and is at $x^{(1)}$, the theorem states that an agent picks $x^{(2)}$ to simply be the orthogonal projection onto the positive region of $h_2$. This is because the cost for going from $x^{(1)}$ to $x^{(2)}$ is simply the $l_2$ distance between them, in which case picking $x^{(2)}$ to be the orthogonal projection of $x^{(1)}$ on $h_2$ minimizes that distance. The main contribution and challenge of Theorem~\ref{thm: 2d_charac} are therefore to understand how to set $x^{(1)}$ and what is the minimum amount of effort that an agent expands to do so.

Now let's examine different cases in Theorem~\ref{thm: 2d_charac}. Note that we assumed $h_1 (x^{(0)}) = 0$ and $h_2 (P_{w_1} (x^{(0)})) = 0$, i.e. that an agent is not in the positive region for the first test and $P_{w_1} (x^{(0)})$ is not in the positive region for the second test, because otherwise, the solution is trivial. In fact, if $h_1 (x^{(0)}) = 1$, then the solution is simply staying at $x^{(0)}$ for the first test and then projecting orthogonally onto the positive region of $h_2$ to pass the second test:
\begin{align*}
& x^{(1)} = x^{(0)}, \ x^{(2)} = P_{w_2} (x^{(1)}) \\
& \manipseq{ x^{(0)} ,\{h_1,h_2\}} = d_{w_2} (x^{(0)})
\end{align*}
This corresponds to region $R_1$ of agents in Figure~\ref{fig:zigzag-region}. If $h_1 (x^{(0)}) = 0$, but $h_2 (P_{w_1} (x^{(0)})) = 1$, then the best response solution is simply the orthogonal projection onto the positive region of $h_1$:
\begin{align*}
&x^{(2)} = x^{(1)} = P_{w_1} (x^{(0)}) \\
&\manipseq{ x^{(0)} ,\{h_1,h_2\}} = d_{w_1} (x^{(0)})
\end{align*}
This corresponds to region $R_4$ of agents in Figure~\ref{fig:zigzag-region}. Additionally, the first case in the closed-form solutions in Theorem~\ref{thm: 2d_charac} corresponds to the region of the space where agents prefer to travel directly to the intersection of the two classifiers than deploying a zig-zag strategy: this corresponds to region $R_3$ in Figure~\ref{fig:zigzag-region}. The second case corresponds to the region where agents do find that a zig-zag strategy is less costly and gives the algebraic characterization of the optimal zig-zag strategy. This region for an agent is denoted by $R_2$ in Figure~\ref{fig:zigzag-region}. Also, as shown by Figure~\ref{fig:zigzag-region}.b, the zig-zag strategy of agents in $R_2$ has the following geometric characterization: pick $x^{(1)}$ on $h_1$ such that the line passing through $x^{(0)}$ and $x^{(1)}$ has angle $\theta$ with the line perpendicular to $h_1$.

Given a budget $\tau$, agents who can manipulate with a cost of at most $\tau$ to pass the two tests in the sequential setting, i.e. $\{x^{(0)}: \manipseq{ x^{(0)} ,\{h_1,h_2\}} \le \tau \} $ is highlighted in Figure~\ref{fig:2d-manipulation}.b.

We conclude this section by showing that if $\theta \ge \pi / 2$, then agents incur the same cost in the sequential setting as they would under the conjunction setting. In other words, agents can deploy the strategy that they would use if they had to pass the two tests simultaneously. 

\begin{restatable}{theorem}{nozigzag}\label{thm: no-zigzag}
    If $\pi/2 \le \theta < \pi$, then for every agent $x^{(0)}$ there exists optimal strategies $x^{(1)}$ and $x^{(2)}$ s.t. $x^{(1)} = x^{(2)}$, i.e.,
    \[
    \manipseq{ x^{(0)} ,\{h_1,h_2\}} = \manipconj{ x^{(0)} ,\{h_1,h_2\}}.
    \]
\end{restatable}
\begin{proof}
    Let $(x^{(1)}, x^{(2)} = P_{w_2} (x^{(1)}))$ be an optimal strategy of the agent in the sequential setting. Suppose $x^{(1)} \neq x^{(2)}$. We have that
    \begin{align*}
    w_1^\top x^{(2)} &= w_1^\top \left( x^{(1)} - (w_2^\top x^{(1)}) w_2 \right) \\
    &= w_1^\top x^{(1)} - (w_2^\top x^{(1)}) (w_1^\top w_2)
    \end{align*}
    But note that $w_1^\top x^{(1)} \ge 0$ because $x^{(1)}$ passes the first classifier by definition, $w_2^\top x^{(1)} \le 0$ because $x^{(1)} \neq x^{(2)}$, and $w_1^\top w_2 \ge 0$ because $\pi/2 \le \theta < \pi$. Therefore, $w_1^\top x^{(2)} \ge 0$ which implies $h_1 (x^{(2)}) = 1$. However, if $h_1 (x^{(2)}) = 1$, then the following manipulation: $y^{(0)} = x^{(0)}$ and $y^{(1)} = y^{(2)} = x^{(2)}$ passes both tests and that its cost is: $\Vert x^{(2)} - x^{(0)} \Vert_2 \le \Vert x^{(2)} - x^{(1)} \Vert_2 + \Vert x^{(1)} - x^{(0)} \Vert_2$ by the triangle inequality. Given the optimality of $(x^{(1)}, x^{(2)})$, we conclude that $(y^{(1)}, y^{(2)})$ is another optimal strategy that the agent can deploy.
\end{proof}

\subsection{Monotonicity}
\label{subsec-monotone}
We now consider a monotonicity property that excludes the possibility of a zig-zag strategy arising. 
A similar property
is noted in 
\citep{socialcost18}.

\begin{definition}[Feature Monotone Classifiers]
Classifier $h_i:\mathbb{R}^d \rightarrow \{0,1\}$ is  \textit{monotone} if for every  individual $x$ that is classified as positive by $h_i$, any feature-wise increase in the features of $x$ results in a positive classification by $h_i$. 
Formally, 
\[
\forall x\in \mathbb{R}^d: h_i(x)=1 \Rightarrow h_i(x+\alpha)=1\quad \forall \alpha\in (\mathbb{R}_{\geq 0})^d.
\]
\end{definition}

Note that this monotonicity property may not hold in some classification problems. For example, most mortgage loans in the US require a good credit score. A common way of improving one's credit score is by getting a credit card and having monthly statements with a balance  greater than zero but not too close to the total credit limit (and paying them on time).

\begin{restatable}{theorem}{NEthm}\label{thm:ne_thm}
 Let $h_1,\ldots,h_k$ be a sequence of monotone classifiers, and let the initial  feature vector $x^{(0)}$ be such that $h_i(x^{(0)})=0$ for every $i\in[k]$. Assume the cost function can be written as $c(x,\hat{x}) = \Vert \hat{x} - x \Vert$ for some norm $\Vert . \Vert$. Then, we have that
 \[
 \manipseq{ x^{(0)} ,\{h_1,\ldots, h_k\}} = \manipconj{ x^{(0)} ,\{h_1,\ldots, h_k\}}.
 \]
\end{restatable}
\begin{proof}
Let $f_{1,\ldots,k}:\mathbb{R}^d\rightarrow \{0,1\}$ denote the function that returns the conjunction of all the classifiers, i.e., $f_{1,\ldots,k}(x)=h_1(x)\land \ldots \land h_k(x)$. 

Let $z^*_{1,\ldots,k}(x^{0})$ denote the point on $f_{1,\ldots,k}$ that minimizes the cost, i.e., $$z^*_{1,\ldots,k}(x^{0})=\text{argmin}_{x^{(1)}}\Vert x^{(0)},x^{(1)}\Vert_p.$$ Note that by definition,  points on $f_{1,\ldots,k}$ are classified as positive by all classifiers $h_1,\ldots,h_k$ (i.e., $z^*_{1,\ldots,k}(x^{0})$ this is the best response for the conjunction case). 

It follows from the triangle inequality that any $x^{(1)}$ such that $h_1(x^{(1)})\land \ldots \land h_k(x^{(1)})=1$ has cost $c(x^{(0)},x^{(1)})\geq c(x^{(0)},z^*_{1,\ldots,k}(x^{0}))$.

We proceed by induction on the number of classifiers. 
For the induction base, consider $k=1$. Clearly, in this case moving to $z^*_{1,\ldots,k}(x)$ yields the best response. 

For the induction step, assume that for every initial point $x'$, and every $k-1$ monotone classifiers $h_2,\ldots,h_k$ it holds that
\[
    \Vert x'-z^*_{2,\ldots,k}(x')\Vert_p\leq \Vert x'-z_2\Vert_2+\ldots+\Vert z_{k-1}-z_{k}\Vert_p.
\]
for every $z_2,\ldots,z_k\in \mathbb{R}^d$ such that $h_i(z_i)=1$.

Adding the additional classifier in the beginning, $h_1$ and considering the initial point, $x$.
Assume by contradiction that there exists a path $x=z_0,z_1\ldots, z_k$ such that $h_i(z_i)\geq 0$ for every $i\in [k]$ and that 
\begin{align}\label{eq:monContra}
    c_{seq}^*(x,\{h_1,\ldots,h_k\}) 
    = \Vert x-z_1\Vert_p+\ldots+\Vert z_{k-1}-z_k\Vert_p 
    <\Vert x-z^*_{1,\ldots,k}(x)\Vert_p.
\end{align}
Since the path from $z_1$ to $z_k$ is a best response for $h_2,\ldots,h_k$ when the initial feature vector $z_1$, by setting $x'=z_1$ we can apply the induction step we and replace this path by $x,z_1,z^*_{2,\ldots,k}(x')$ without increasing the sum of manipulations. 
If $f_{1,\ldots,k}(z^*_{2,\ldots,k}(z_1))=1$, we have that $\Vert x- z_1\Vert_p+\Vert z_1- z^*_{2,\ldots,k}(z_1)\Vert_p\leq \Vert x-z^*_{1,\ldots,k}(x)\Vert_p$ due to the triangle inequality and the definition of $z^*_{1,\ldots,k}(x)$ and this is a contradiction to Eq.~\ref{eq:monContra}.

So assume $f_{1,\ldots,k}(z^*_{2,\ldots,k}(z_1))=0$.
Since $h_i(z^*_{2,\ldots,k}(z_1))=1$ for every $i\geq 2$ by definition, we have that $h_1(z^*_{2,\ldots,k})=0$. As $h_1(z_1)=1$, we can define $z'\in \mathbb{R}^d$ such that 
\[ 
z'[j]=\max\{{z^*_{2,\ldots,k}(z_1)[j], z_1[j]}\}, 
\]
and from monotonicity it follows that $f_{2,\ldots,k}(z')=1$.

Finally, we have that $\Vert x-z_1\Vert_p+ \Vert z_1-z'\Vert_p < \Vert x-z_1\Vert_p+ \Vert z_1-z^*_{2,\ldots,k}(z_1)\Vert_p$, which is a contradiction to the minimiality of $z^*_{2,\ldots,k}(z_1)$ and thus to the minimality of $z_2,\ldots,z_k$.
\end{proof}
Theorem~\ref{thm:ne_thm} in particular implies that under our monotonicity assumption and for a large class of reasonable cost functions, an agent has no incentive to zig-zag in the sequential case and in fact can simply follow the same strategy as in the simultaneous or conjunctive case. This insight immediately extends even when $x^{(0)}$ is positively classified by some but not all of the $h_i$'s as any best response is guaranteed to increase the feature values and thus will maintain the positive classification results of these classifiers. 

\section{Manipulation Resistant Defenses}\label{sec:defense}
Up to this point in the paper, we have focused mainly on the existence and feasibility of a zig-zag manipulation strategy from the perspective of an agent. 
We now shift gears and discuss the firm's decision space. We are interested in understanding how the firm can modify its classifiers to maintain a high level of accuracy (if possible), despite the strategic manipulations of an agent. To this end, we assume there is a joint distribution of features and labels $\mathcal{D}$ over $\mathcal{X}\times \{0,1]\}$. 
Interestingly, previous works \citep{bruckner2011stackelberg, hardt2016strategic} show hardness results for finding optimal strategic classifiers, where the objective is finding a single classifier $h$ that attains the strategic maximum accuracy.

Now, we can introduce the defender's game for a typical strategic classification problem.
\begin{align}
\begin{split}
\min_{h \in \mathcal{H}}&~~~P_{(x,y) \sim \mathcal{D}} [h(z^{*} (x)) \neq y ] \\
\text{s.t.}&~~~z^{*}(x) = \arg\max_{z}~h(z) - c(x,z) 
\end{split}
\end{align}
In our paper, $h$ is actually given by the sequential composition of classifiers in the screening process and $c(x,z)$ is the sum of manipulation costs per stage. The objective function in this optimization problem is a direct generalization of $0$-$1$ loss for normal learning problems, only complicated by the strategic behavior of an agent. 

As \citet{bruckner2011stackelberg} observe, this is a bi-level optimization problem and is NP-hard \citep{Jeroslow85} to compute, even when constraints and objectives are linear. 
Interestingly, \citet{hardt2016strategic} also show a hardness of approximation result for general metrics. Because of these past hardness results, we instead focus on a more tractable defense objective. 

\subsection{Conservative Defense}\label{sec:conservative-defense}
Here, we consider a different objective motivated by the hiring process in firms, in which avoiding false positives and not hiring unqualified candidates can  be seen as arguably more important than avoiding false negatives and not missing out on good candidates. This objective, described below, has been previously studied in the context of strategic classification, in particular in~\citep{gameimprove}. 
\begin{definition}[No False Positive Objective]
Given the manipulation budget $\tau$ and the initial linear classifiers $h_1, \cdots, h_k$, the goal of the firm is to design a modified set of linear classifiers $\tilde{h}_1, \cdots, \tilde{h}_k$ that maximize the true positive rate of the pipeline on manipulated feature vectors subject to no false positives. Recall that the ground truth is determined by the conjunction of $h_1, \cdots, h_k$ on unmanipulated feature vectors of agents.   
\end{definition}
Without loss of generality, we assume the pipeline is non-trivial: the intersection of acceptance regions of $h_1, \cdots, h_k$ is non-empty. 

We prove that, under standard assumptions on linear classifiers of the firm, a defense strategy that ``shifts" all classifiers by the manipulation budget, is the optimal strategy for the firm in both pipeline and conjunction settings. We formally define the defense strategy as follows: 
\begin{definition}[Conservative Strategy]\label{def:conservative-defense}
Given the manipulation budget $\tau$, the firm conservatively assumes that each agent has a manipulation budget of $\tau$ per test. For each test $h_i(x) = \mathbbm{1}[w_i ^\top x \geq b_i]$, the firm replaces it by a ``$\tau$-shifted" linear separator $\Tilde{h}_i(x) = \mathbbm{1}[w_i ^\top x \geq b_i + \tau])$. In this section, without loss of generality, we assume that all $w_i$'s have $\ell_2$-norm equal to one. 
\end{definition}

Our statement holds when the linear classifiers satisfy the following ``general position" type condition. \begin{definition}\label{def:gen-position}
We say a collection of linear classifiers $\mathcal{H} = \{h_1(x) = \mathbbm{1}[w_1 ^\top x \geq b_1], \cdots, h_k(x) = \mathbbm{1}[w_k ^\top x \geq b_k]\}$ with $w_1, \cdots, w_k\in \mathbb{R}^d$ are in ``general position" if for any $i\in [k]$, the intersection of $\{x | w_i^\top x = b_i\}$ and $\{x | \bigwedge_{j\in [k], j \neq i} h_j(x)=1\}$ lies in a $(d-1)$-dimensional subspace but in no $(d-2)$-dimensional subspace. We remark that in $\mathbb{R}^2$, this condition is equivalent to the standard general position assumption (i.e., no three lines meet at the same point). Moreover, this condition implies that no test in $\mathcal{H}$ is ``redundant", i.e., for every $i\in[k]$, the positive region of $\mathcal{H}$ (i.e., $\bigwedge_{h\in \mathcal{H}} \{x|h(x)=1\}$) is a proper subset of the positive region of $\mathcal{H}\setminus{h_i}$.   See Figure~\ref{fig:general-assumption} for an example in $\mathbb{R}^2$.
\end{definition} 
\begin{figure}
    \centering
    \includegraphics{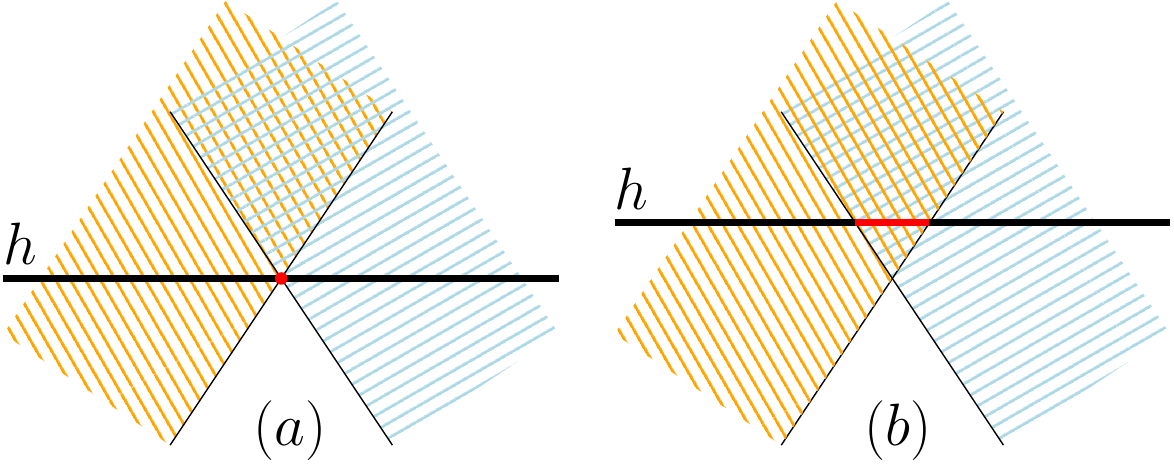}
    \caption{In ($a$), the intersection of $h$ with the positive half plane of the other two classifiers that are in blue and gray shadows is a point which is of zero dimension. This case is not in the general position and $h$ is a redundant classifier. However, in ($b$), the intersection of $h$ with the described positive regions is a line segment, a one-dimensional object. Here, $h$ is not redundant.}
    \label{fig:general-assumption}
\end{figure}
Now, we are ready to state the main result of this section.
\begin{restatable}{theorem}{thmConservativeDef} \label{thm:conservative-defense-optimality}
Consider a set of linear classifiers $\mathcal{H}= \{h_1, \cdots, h_k\}$ that are in ``general position" (as in Definition~\ref{def:gen-position}). Moreover, suppose that each agent has a manipulation budget of $\tau$. Then, in both the conjunction and sequential settings, the conservative defense is
a strategy that maximizes true positives subject to zero false positives.
\end{restatable}
\begin{proof}
First, we prove that conservative defense achieves zero false positive in both cases. To show this, by Claim~\ref{claim:conjBoundSeq}, it suffices to show it for the sequential setting only. Consider an agent $x$ who initially (i.e., before manipulation) is not in the positive region of conjunctions of $h_1, \cdots h_k$; i.e., $\Pi_{j\in [k]}h_j(x^{(0)})=0$. Hence, there exists a classifier $h_i$ such that $w^\top_i x^{(0)}< b_i$. Now, let $x^{(i)} : x^{(0)} + \epsilon_i$ denote the (manipulated) location of $x$ right before stage $i$. Since the total manipulation budget of $x$ is $\tau$, $w_i^\top x^{(i)} \le w_i^\top x^{(0)} + w_i^\top \epsilon_i < b_i + \tau$ (the choice of $\varepsilon_i$ that maximizes $w_i^\top \epsilon_i$ is $\epsilon_i = \tau w_i$, and $w_i^\top (\tau w_i) = \tau$ since $\Vert w_i \Vert_2 = 1$). Hence, $\tilde{h}(x^{(i)}) = 0$ and agent $x$ cannot pass the modified pipeline $\Tilde{h}_1, \cdots, \Tilde{h}_k$. 

Next, consider test $i$ and let $\Delta^i$ denote the subspace of points (i.e.,~agents) in the intersection of $\{x | h_i(x)=0\}$ and $ \bigwedge_{j\in [k], j\neq i} \{x| h_j(x)=1\}$. By the general position assumption, $\Delta^i$ is a $(d-1)$-dimensional subspace and is a subset of the $(d-1)$-dimensional hyperplane corresponding to $w^\top_i x = b_i$. Then, there exists only a unique linear separator which is at distance exactly $\tau$ from $\Delta^i$ (and is in the positive side of $h_i$); $\hat{h}_i(x) := \mathbbm{1}[w^\top_i x \ge b_i + \tau]$. Given that any defense strategy with zero false positive has to classify an agent in $\Delta^i$ as negative, it is straightforward to verify that any ``feasible" modified linear separator $h'_i$ (i.e., achieving zero false positive) results in true positive rate less than or equal to the one replaces $h'_i$ with $\hat{h}_i$.
\end{proof}

Note that while the conservative defense strategy has the maximum possible true positive subject to zero false positive in both simultaneous and sequential settings, by Claim~\ref{claim:conjBoundSeq}, the conservative defense achieves a higher true positive rate in the sequential setting compared to the simultaneous case. Informally, from the firm's point of view, {\em under manipulation, the sequential setting is a more efficient screening process}.

\section{Discussion}
We have initiated the study of \textit{Strategic Screening}, combining screening problems with strategic classification. 
This is a natural and wide-spread problem both in automated and semi-automated decision making. 
We believe these examples and our convex program  can aid in the design and monitoring of these screening processes. 

Substantial open questions remain regarding fairness implications of the defender's solution and exactly how susceptible real world pipelines are to zig-zagging. Some of the works cited in the related work section consider fairness considerations in the space of strategic manipulation, stemming either from unequal abilities to manipulate~\citep{socialcost18,hu2019disparate} or unequal access to information about the classifiers~\citep{bechavod2022information} across different groups. We do not consider these connections in our work, but these considerations are of significant interest and a natural direction for further research, especially due to the importance of making fair decisions in high-stake, life altering contexts. We finish with a few interesting examples for this.

Disparities might arise both in the conjunction and in the sequential setting, with or without defense. 
Consider the classifiers presented in Example \ref{exp:zigzag} and an instance in which candidates belong to two groups, $G^1$ and $G^2$ with initial feature vector distributed identically and characterized by different total manipulation budgets,  $\sqrt{2}=\tau^2>\tau^1=5/4$. 
The narrative of the fairness disparities in the conjunction case is a simple generalization of the single classifiers case (e.g., \citep{hardt2016strategic})- If the distribution is such that a significant fraction of individuals (from both groups) starts at a feature vector that is classified by both classifiers as $0$ and that requires $\sqrt{2}$ manipulation cost to reach their intersection--- only the individuals form $G_2$ will be able to manipulate. 
For the sequential case, consider a distribution with a large enough fraction of individuals starting at $(0,0)$. Example~\ref{exp:zigzag} demonstrates that only individuals from $G_2$ will have sufficient budget to manipulate (using the zig-zag strategy). 
If the firm applies the conservative defense, individuals from $G_1$ that should have been classified as positive might not have sufficient  budget to manipulate their way to acceptance,  which in turn implies  higher false negative rates. 
This indicates, similarly to prior results in strategic classification (e.g.,~\citep{hu2019disparate}), how the  members of the advantaged group are more  easily admitted or hired. 

\section*{Acknowledgements}
The authors are very grateful to Avrim Blum and Saba Ahmadi for  helpful comments on an earlier draft and discussion of related work in the literature.

\bibliographystyle{abbrvnat}
\bibliography{strategic-pipeline}

\end{document}